\documentclass[10pt, two column, twoside]{IEEEtran}

\usepackage{amssymb}
\usepackage{amsmath}
\usepackage[lined,boxed,commentsnumbered, ruled]{algorithm2e}
\usepackage{mathrsfs}
\usepackage{algorithmic}
\usepackage{bm}
\usepackage{tikz}
\usepackage{pgfplots}
\usetikzlibrary{shapes}
\usepackage{subfig}
\usepackage{graphicx,booktabs,multirow}
\usepackage{enumerate} 

\definecolor{color1}{HTML}{D0B22B}
\definecolor{dred}{RGB}{128,0,0}
\definecolor{colorhkust}{RGB}{20,43,140}
\definecolor{colorshanghaitech}{RGB}{162,0,5}
\definecolor{colortsinghua}{RGB}{116,52,129}
\definecolor{colordark}{RGB}{184,134,11}
\usepackage{amsthm}
\theoremstyle{definition}

\newtheorem{proposition}{Proposition}
\newtheorem{definition}{Definition}

\newcommand{\bs}[1]{{\bm{#1}}}

\begin{document}

\title{Federated Learning via Over-the-Air Computation}

\author{Kai~Yang,~\IEEEmembership{Student Member,~IEEE,} Tao ~Jiang,~\IEEEmembership{Student Member,~IEEE,} Yuanming~Shi,~\IEEEmembership{Member,~IEEE,}
        and~Zhi~Ding,~\IEEEmembership{Fellow,~IEEE}

\thanks{K. Yang is with the School of Information Science and Technology, ShanghaiTech University, Shanghai 201210, China, and also with the University of Chinese Academy of Sciences, Beijing 100049, China (e-mail: yangkai@shanghaitech.edu.cn).}
\thanks{T. Jiang and Y. Shi are with the School of Information Science and Technology, ShanghaiTech University, Shanghai 201210, China (e-mail: \{jiangtao1,shiym\}@shanghaitech.edu.cn).}
\thanks{Z. Ding is with the Department of Electrical and Computer Engineering,
University of California at Davis, Davis, CA 95616 USA (e-mail:
zding@ucdavis.edu).}
}
       
\maketitle
\IEEEpeerreviewmaketitle

\begin{abstract}
The stringent requirements for low-latency and privacy of the emerging  high-stake applications with intelligent
devices such as drones and smart vehicles make the cloud computing inapplicable in these scenarios. Instead,  \textit{edge machine learning} becomes increasingly attractive for performing training and inference directly at network edges without sending data to a centralized data center. This stimulates a nascent field termed as \textit{federated learning} for training a machine learning model on computation, storage, energy and bandwidth limited mobile devices in a distributed manner. To preserve data privacy and address the issues of unbalanced and non-IID data  points across different devices, the federated averaging algorithm has been proposed for global model aggregation by computing the weighted average of locally updated model at each selected device. However, the limited communication bandwidth becomes the main bottleneck for aggregating the locally computed updates. We thus propose a novel  \textit{over-the-air computation} based approach for fast global model aggregation via exploring the superposition property of a wireless multiple-access channel. This is achieved by joint  device
selection and beamforming design, which is modeled as a sparse and low-rank optimization problem to support efficient algorithms design. To achieve this goal, we provide  a difference-of-convex-functions (DC) representation for the sparse and low-rank function to enhance sparsity and accurately detect the fixed-rank constraint in the procedure of device selection. A DC algorithm is further developed to solve the resulting DC program with global convergence guarantees.  The algorithmic advantages and admirable performance of the proposed methodologies are demonstrated through extensive numerical results.   

\end{abstract}

\section{Introduction}
The astounding growth in data volume promotes widespread artificial intelligent applications such as image recognition and natural language processing \cite{lecun2015deep}, thanks to the recent breakthroughs in machine learning (ML) techniques particularly deep learning, as well as the unprecedented levels of computing power \cite{stoica2017berkeley}. Nowadays the typical machine learning procedure including the  training process and the inference process, is supported by the cloud computing, i.e., a centralized cloud data center with the broad accessibility of computation, storage and the whole dataset. However, the emerging intelligent mobile devices and high-stake applications such as drones, smart vehicles and augmented reality, call for the critical requirements of low-latency and privacy. This makes the cloud computing based ML methodologies inapplicable \cite{zhu2018towards}. Therefore, it becomes increasingly attractive to possess data locally at the edge devices and then performing training/inference directly at the edge, instead of sending data to the cloud or networks. This emerging technique  is termed as \textit{edge ML} \cite{park2018edgeai}. The main bottleneck is the limited computation, storage, energy and bandwidth resources to enable mobile edge intelligent services. To address this issue, there is a growing body of recent works to reduce the storage overhead, time and power consumption
in the inference process using the model compression methods via hardware and software co-design \cite{han2015deep,
Zhang_SPM18}. Furthermore, various advanced distributed optimization algorithms \cite{lin2017deep, wang2018edge, karakus2017straggler, mcmahan2017communication, smith_nips2017federated} have been proposed to speed up the training process by taking advantages of the computing power and distributed data over multiple
devices.  

Recently, a nascent field called \textit{federated learning} \cite{mcmahan2017communication, smith_nips2017federated,bonawitz2019towards,yang2019federated,tran2019federatedi} investigates the possibility of distributed learning directly on the mobile devices to enjoy the benefits of better privacy and less network bandwidth. However, a number of challenges arise to deploy the federated learning technique. 1) The collected non-IID data across the network (i.e.,
the data is generated by distinct distributions across different devices), imposes significant statistical challenges to fit a mode from the non-IID data \cite{smith_nips2017federated, zhao2018federated}. 2) Large communication loads across mobile devices limit the scalability for federated learning to efficiently exchange
locally computed updates at each device \cite{mcmahan2017communication, wang2017giant}. 3) The heterogeneity
of computation, storage and communication capabilities across different devices
brings unique system challenges to tame latency for on-device distributed training, e.g., the stragglers (i.e., devices that run slow) may cause significant delays \cite{wang2018edge, li2018near}. 4) The arbitrarily adversarial behaviors of the devices (e.g,. Byzantine failures \cite{blanchard2017machine})\ bring critical security
issues for large-scale distributed learning, which will incur a major degradation of the learning performance \cite{chen2017distributed}. 5) System implementation issues such as the unreliable device connectivity, interrupted execution and slow convergence compared with learning on centralized data \cite{bonawitz2019towards}. In particular, the federated averaging
(FedAvg)  algorithm \cite{mcmahan2017communication} turns out to be a promising way to efficiently average the locally
updated model at each device with unbalanced and non-IID data, thereby reducing the number of communication rounds between the center node and the end devices.        

In this paper, we focus on designing the fast model aggregation approach for the FedAvg algorithm to improve the communication efficiency  and speed up the federated
learning system. We observe that the global model aggregation procedure consists of the transmission of locally computed updates from each device, followed by the computation of their weighted average at a central node. We shall propose a computation and communication co-design approach for fast model aggregation by leveraging the principles of \textit{over-the-air computation} (AirComp) \cite{nazer2007computation}. This is achieved by exploring the superposition property of a wireless multiple-access channel to compute the desired function (i.e., the weighted average function) of distributed locally computed updates via concurrent transmission. Although the AirComp problem has achieved significant progresses from the point of view of information theory \cite{nazer2007computation}, signal processing \cite{Goldenbaum_TSP13harnessing} and  transceiver beamforming design \cite{Kaibin_IoT2018mimo}, the AirComp based model aggregation problem brings unique challenges as we need to simultaneously minimize
the
function distortion and maximize the number of involved devices. This is based on the key observations that the aggregation errors may lead to a notable drop
of the prediction accuracy, while the convergence
of training can be accelerated with more involved devices \cite{mcmahan2017communication,wang2018cooperative}. To improve
the communication efficiency and statistical performance of federated learning, we shall propose a joint device selection and receiver beamforming design approach to find the maximum selected devices with the mean-square-error (MSE) requirement for fast model aggregation via AirComp. Note that the tradeoff of learning performance and aggregation error is also considered in the recent parallel work \cite{zhu2018low}, which quatifies the device population of the truncation-based approach for excluding the devices with deep fading channel.
 
However, the joint device selection and beamforming design problem is essentially a computationally difficult mixed combinatorial optimization problem with nonconvex quadratic constraints. Specifically, device selection needs to maximize a combinatorial objective function, while the MSE requirement yields nonconvex quadratic constraints due to the multicasting duality for receiver beamforming design in AirComp \cite{Kaibin_IoT2018mimo}. To address the computational issue, we propose a sparse and low-rank modeling approach to assist efficient algorithms design. This is achieved by finding a sparse representation for the combinatorial objective function, followed by reformulating the nonconvex quadratic constraints as affine constraints with an additional rank-one matrix constraint by adopting the matrix lifting technique \cite{luo2007approximation}.
For the sparse optimization problem, $\ell_1$-norm is a celebrated convex surrogate for the nonconvex $\ell_0$-norm. The nonconvex smoothed $\ell_p$-norm supported by the iteratively reweighted algorithm is a promising way to enhance the sparsity level \cite{shi2016smoothed, Yuanming_cvxsmooth18}. However, its convergence results rely on the carefully chosen smoothing parameter. Although the semidefinite relaxation (SDR) technique convexifies the  nonconvex quadratic constraints as a linear constraint via dropping the rank-one constraint in the lifting problem, the performance degenerates with large number of antennas as its weak capability of inducing low-rank structures \cite{chen2018uniform}.

To address the limitations of existing algorithms for solving the presented sparse and low-rank optimization problem, we propose a unified difference-of-convex-functions (DC) approach to induce both the sparsity and low-rank structures. Specifically, to enhance sparsity, we adopt a novel DC representation for the $\ell_0$-norm \cite{gotoh2017dc}, which is given by the difference of the $\ell_1$-norm
and the Ky Fan $k$-norm \cite{fan1951maximum}, i.e., sum of the largest $k$
absolute values. We also provide a DC representation for the rank-one constraint of the positive semidefinite
matrix by setting the difference between its trace norm and  spectral norm
as zero.  Based on the novel DC representations for the sparse function and low-rank constraint, we propose to induce the sparse structure in the first step as a guideline for the priority of selecting devices. In the second step, we solve a number of feasibility detection problems to find the maximum selected devices via accurately satisfying the rank-one constraint. Our proposed DC approach for enhancing sparsity is parameter free. The exact detection of the rank-one constraint is critical for accurately detecting the feasibility of nonconvex quadratic constraints in the procedure of device selection. Furthermore, the computationally efficient DC Algorithm (DC) with global convergence guarantee is developed by successively solving the convex relaxation of primal problem and dual problem of the DC program.
These algorithmic advantages make the proposed DC approach for sparse and low-rank optimization outperform state-of-the-art approaches considerably. 

\subsection{Contributions}

In this paper, we propose a novel over-the-air computation approach to enable fast global model aggregation for on-device distributed federated learning via harnessing the signal superposition property of a wireless multiple-access
channel. To improve the statistical learning
performance and the convergence rate for on-device distributed learning, we propose
to maximize the number of involved devices for global model aggregation
while satisfying the MSE requirement to reduce the model aggregation error.  This is achieved by joint device selection and beamforming design, which is further modeled as a sparse and low-rank optimization problem. A novel DC approach is developed to enhance sparsity and accurately detect rank-one constraint. The DC algorithm with established convergence rate is further developed via successively convex relaxation.

The main contributions of the paper are summarized as follows:
\begin{enumerate}
    \item We design a novel fast model aggregation approach for federated learning via exploiting signal superposition property of a wireless multiple-access
channel using the principles of over-the-air computation. This idea is achieved
by joint device selection and  beamforming design to improve the statistical learning performance.   
    \item A sparse and low-rank modeling approach is provided to support efficient algorithms design for the joint device selection and beamforming problem, which is essentially a highly intractable combinatorial optimization problem with nonconvex quadratic constraints.
    \item To address the limitations of existing algorithms for sparse and low-rank optimization, we propose a unified DC representation approach to induce both the sparse and low-rank structures. The proposed DC approach has the capability of accurately detecting the feasibility of nonconvex quadratic constraints, which is critical in the procedure of device selection.
    \item We further develop a DC algorithm for the presented nonconvex DC program via successive convex relaxation. The global convergence rate of the DC algorithm is further established  by rewriting the DC   function as the difference of strongly convex functions.
\end{enumerate}
The superiority of the proposed DC approach for accurately feasibility detection and device selection will be demonstrated through extensive numerical results. It turns out that our proposed approaches can achieve better prediction accuracy and faster convergence rate in the experiments of training support vector machine (SVM) classifier on CIFAR-10 dataset.

\subsection{Organization}
The remaining part of this work is organized as follows. Section II introduces the system model of on-device distributed federated learning and problem formulation for fast model aggregation. Section III presents a sparse and low-rank modeling approach for model aggregation. Section IV provides the DC representation framework for solving the sparse and low-rank optimization problem, while in Section V the DC Algorithm is developed and its convergence rate is also established. The performances of the proposed approaches and other state-of-the-art approaches are illustrated in Section VI. We conclude this work in Section VII.

\section{System Model and Problem Formulation}\setlength\arraycolsep{1.5pt}
In this section, the on-device distributed federated learning system is presented. Based on the principles of over-the-air computation, we propose a computation and communication co-design approach based on the principles of over-the-air computation for fast model aggregation of locally computed updates at each device to improve the global model.

\subsection{On-Device Distributed Federated Learning}
\label{fedlearning}
As an on-device distributed training system, federated learning keeps the training data at each device and learns a shared global model from distributed mobile devices. With this novel distributed learning paradigm, lots of benefits can be harnessed such as low-latency, low power consumption as well as preserving users' privacy \cite{mcmahan2017communication}. Fig. \ref{fig:FedAvg} illustrates the federated learning system with $M$ single-antenna mobile devices and one computing enabled base station (BS) equipped with $N$ antennas to support the following distributed machine learning task:
\setlength\arraycolsep{2pt}
\begin{equation}
        \mathop{\textrm{minimize}}_{\bs{z}\in\mathbb{R}^d}\quad f(\bs{z})=\frac{1}{T}\sum_{i=1}^T f_i(\bs{z}),
\end{equation}
where $\bm{z}$ is the model parameter vector to be optimized with dimension $d$ and $T$ is the total number of data points. This model is widely used in linear regression, logistic regression, support vector machine, as well as deep neural networks. Typically, each function $f_i$ is parameterized by $\ell(\bs{z};\bs{x}_i,y_i)$, where $\ell$ is a loss function with the input-output data pair as $(\bs{x}_i,y_i)$. Here, $\mathcal{D}=\{(\bs{x}_i,y_i):i=1,\cdots,T\}$ denotes the dataset involved in the training process. The local dataset at device $k$ is denoted as $\mathcal{D}_k\subseteq\mathcal{D}$.
\begin{figure}[h]
        \centering
        \includegraphics[width=0.85\columnwidth]{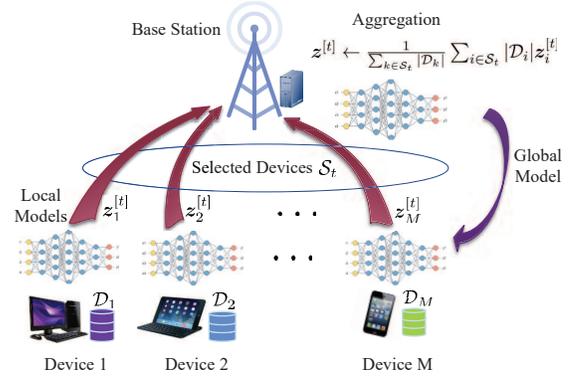}
        \caption{On-device distributed federated learning system.}
        \label{fig:FedAvg}
\end{figure}

A recognized problem for on-device distributed federated learning system is the limited network bandwidth, which becomes the main bottleneck for globally aggregating the locally  computed updates at each mobile device. To reduce the number of communication rounds between mobile devices and the BS for global model updating, the federated averaging (FedAvg) algorithm \cite{mcmahan2017communication} has recently been proposed, which is also referred to as model averaging. Specifically, at the $t$-th round:
\begin{enumerate}[1)]
    \item The BS selects a subset of mobile devices $\mathcal{S}_t\subseteq\{1,\cdots,M\}$;
    \item  The BS sends the updated global model $\bs{z}^{[t-1]}$ to the selected devices $\mathcal{S}_t$; 
    \item  Each selected device $k\in\mathcal{S}_t$ runs a local update algorithm (e.g., stochastic gradient algorithm) based on its local dataset $\mathcal{D}_k$ and the global model $\bs{z}^{[t-1]}$,  whose output is the updated local model $\bs{z}^{[t]}_{k}$; 
    \item The BS aggregates all the local updates $\bs{z}^{[t]}_{k}$ with $k\in\mathcal{S}_t$, i.e., computing their weighted average as the updated global model $\bs{z}^{[t]}$.
\end{enumerate}
The federated averaging framework is thus presented in Algorithm \ref{algorithm:FedAvg}.

 {\SetNlSty{textbf}{}{:}
\IncMargin{1em}
\begin{algorithm}[h]
\textbf{BS executes:} \\
 initialize $\bm{w}_0$.\\
 \For{each round $t=1,2,\cdots$}{
 $\mathcal{S}_t\leftarrow$ select a subset of $M$ devices; \\
 broadcast global model $\bs{z}^{[t-1]}$ to devices in $\mathcal{S}_t$. \\
 \For{each mobile device $k\in\mathcal{S}_t$ \textbf{in parallel}}{
 $\bs{z}^{[t]}_k\leftarrow$ LocalUpdate$(\mathcal{D}_k,\bs{z}^{[t-1]})$ 
 }
$\bs{z}^{[t]} \leftarrow \frac{1}{\sum_{k\in\mathcal{S}_t}|\mathcal{D}_k|}\sum_{k\in\mathcal{S}_t}|\mathcal{D}_k|\bs{z}^{[t]}_k$ (\bf{aggregation})
 }
 \caption{Federated Averaging (FedAvg) Algorithm}
 \label{algorithm:FedAvg}
\end{algorithm}}

In this paper, we aim at improving the communication efficiency for on-device distributed federated learning by developing a fast model aggregation approach for locally computed updates in the FedAvg algorithm. A key observation for the FedAvg algorithm is that the statistical learning performance can be improved by selecting more workers in each round \cite{mcmahan2017communication,wang2018cooperative}. As an illustrative example in Fig. \ref{fig:random_selection}, we train an support vector machine (SVM) classifier on the CIFAR-10 dataset \cite{krizhevsky2009learning} with FedAvg algorithm and show the training loss and prediction accuracy over the number of selected devices. The federated learning system consists of $10$ mobile devices in total and the selected devices are chosen uniformly at random for each round. 
However, selecting more devices also brings higher communication overhead for aggregating the local computed updates at each selected device. 
\begin{figure}[h]
        \centering
        \subfloat[Training loss]{\includegraphics[width=0.85\columnwidth]{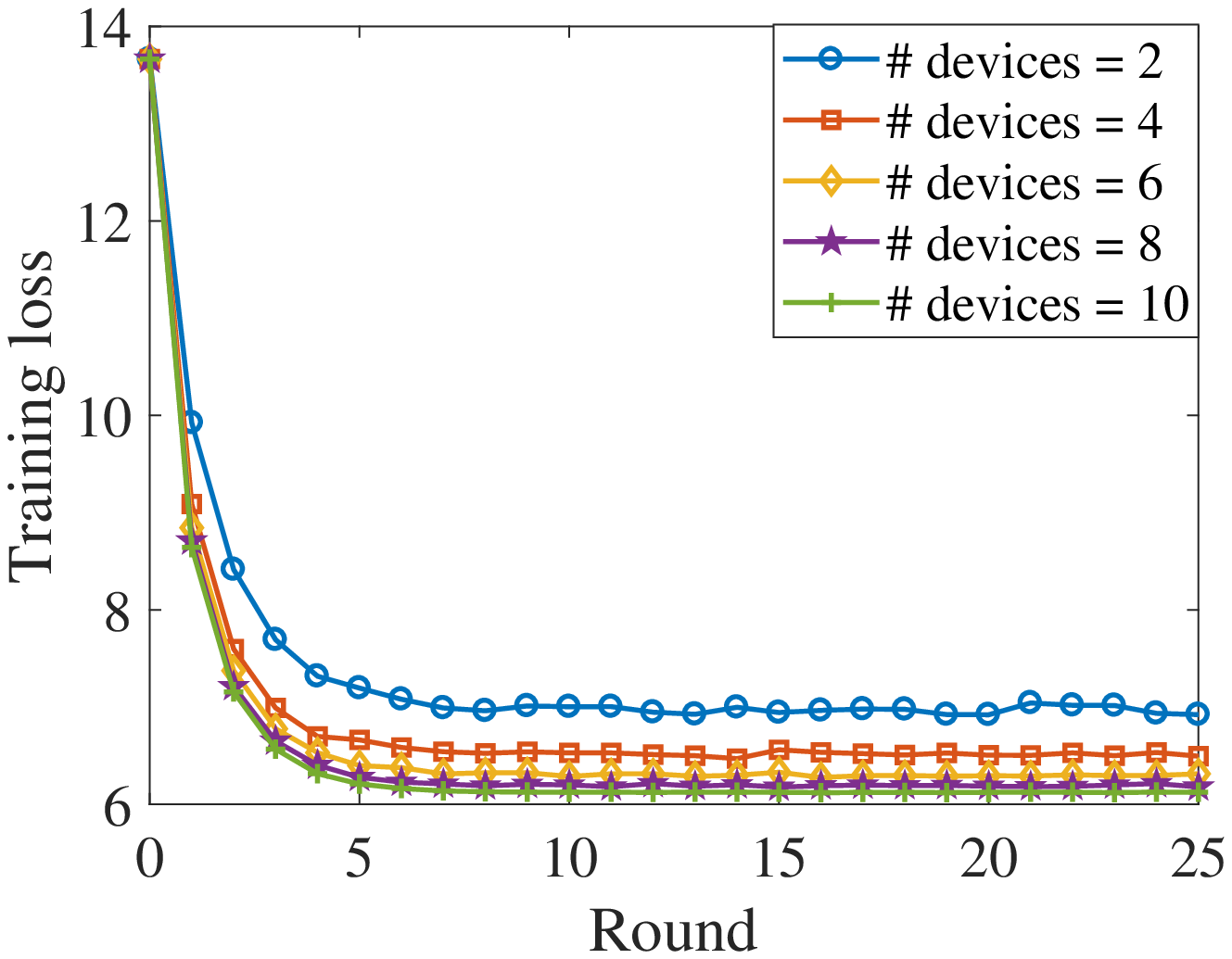}}\hfil
        \subfloat[Relative prediction accuracy]{\includegraphics[width=0.85\columnwidth]{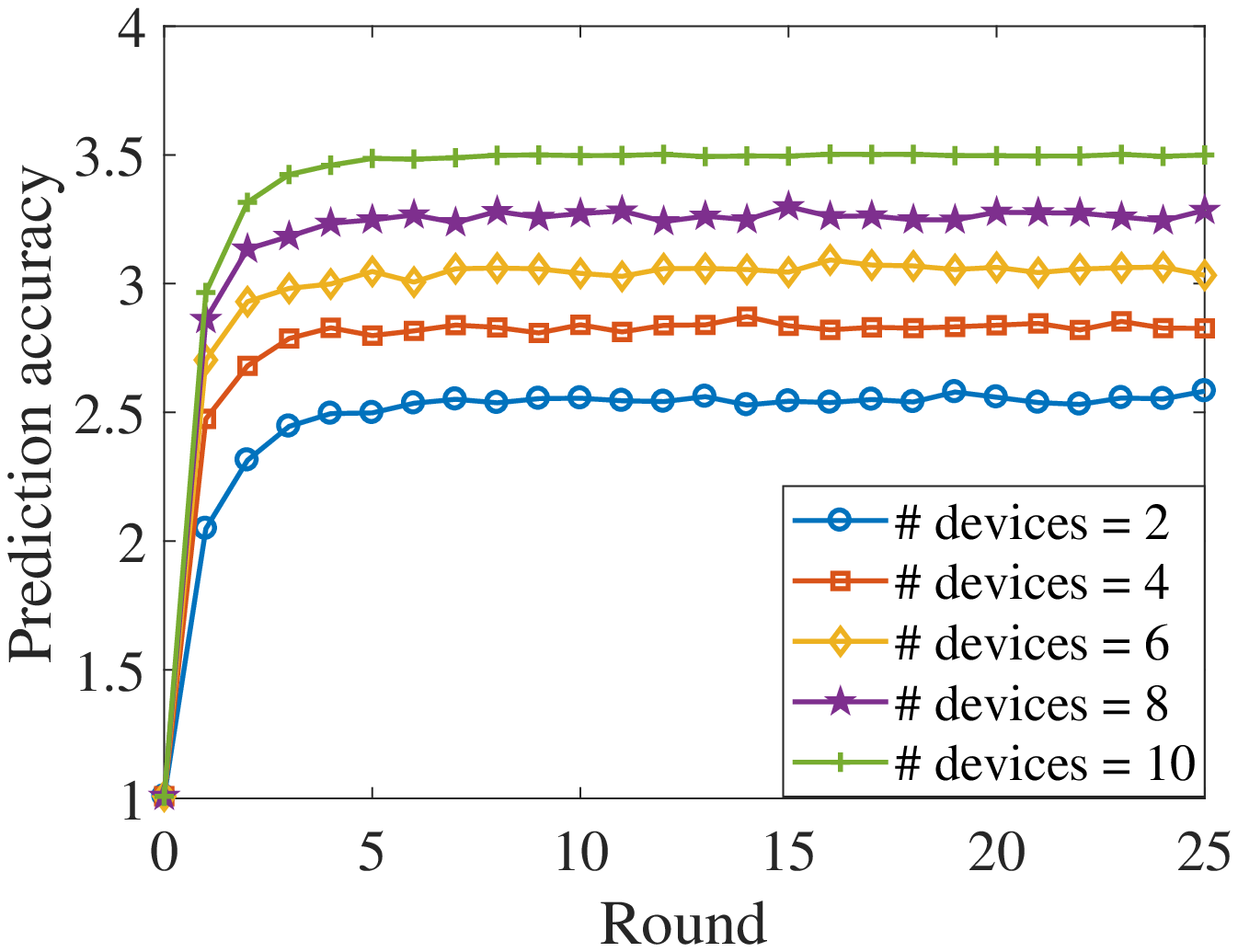}}
        \caption{The training loss and prediction accuracy with different number of randomly selected devices for FedAvg. We train an support vector machine (SVM) classifier on the CIFAR-10 dataset and adopt the stochastic gradient descent algorithm \cite{krizhevsky2009learning} as the local update algorithm for each device. Each curve is averaged for $10$ times. The relative prediction accuracy is defined as the accuracy over random classification.}
        \label{fig:random_selection}
\end{figure}

Note that the model aggregation procedure requires the computation of the weighted average of locally computed updates and the communication from selected mobile devices to the BS. Therefore, in this paper we develop a novel communication and computation co-design approach for fast model aggregation. Our approach is based on the principles of \emph{over-the-air computation} \cite{nazer2007computation} by leveraging the signal superposition property of a multiple-access channel. Furthermore, we notice that the aggregation error also causes a notable drop of the prediction accuracy  \cite{keskar2016large}. To address this issue, we shall develop efficient transceiver strategies to minimize the distortion error for model aggregation via over-the-air computation. Based on the above key observations, in this paper, we focus on the following two aspects to improve the statistical learning performance in on-device distributed federated learning system:
\begin{itemize}
        \item Maximize the number of selected devices at each round to improve the convergence rate in the distributed training process;
        \item Minimize the model aggregation error to improve the prediction accuracy in the inference process.
\end{itemize}

\subsection{Over-the-Air Computation for Aggregation}
Over-the-air computation has become a promising approach for fast wireless data aggregation via computing a nomographic function (e.g., arithmetic mean) of distributed data from multiple transmitters \cite{Goldenbaum_TSP13harnessing}. By integrating computation and communication through exploiting the signal superposition property of a multiple-access channel, over-the-air computation can accomplish the computation of target function via concurrent transmission, thereby significantly improving the communication efficiency compared with orthogonal transmission. The key observation in the FedAvg algorithm  is that the global model is updated through computing the weighted average of locally computed updates at each selected device, which falls in the category of computing nomographic functions of distributed data. In this paper, we shall propose the over-the-air computation approach for communication efficient aggregation in federated learning system.

Specifically, the target vector for aggregating local updates in the FedAvg algorithm is given by 
\begin{eqnarray}\label{eq:target_function}
\bs{z} = \psi\left(\sum_{i\in\mathcal{S}}\phi_i(\bs{z}_i)\right),
\end{eqnarray}
where $\bs{z}_i$ is the updated local model at the $i$-th device, $\phi_i=|\mathcal{D}_i|$ is the pre-processing scalar at device $i$, $\psi=\frac{1}{\sum_{k\in\mathcal{S}}|\mathcal{D}_k|}$ is the post-processing scalar at the BS, and $\mathcal{S}$ is the selected set of mobile devices. The symbol vector for each local model before pre-processing $\bs{s}_i := \bs{z}_i\in\mathbb{C}^{d}$ is assumed to be normalized with unit variance, i.e., $\mathbb{E}(\bs{s}_i\bs{s}_i^{\sf{H}})=\bs{I}$. At each time slot $j\in\{1,\cdots, d\}$, each device sends the signal $s_{i}^{(j)}\in\mathbb{C}$ to the BS. We denote
\begin{equation}\label{eq:target_function}
    g^{(j)}=\sum_{i\in\mathcal{S}}\phi_i\Big(s_i^{(j)}\Big)
\end{equation}
 as the target function to be estimated through over-the-air computation at the $j$-th time slot. 

 To simplify the notation, we omit the time index by writing $g^{(j)}$ and $s_i^{(j)}$ as $g$ and $s_i$, respectively. The received signal at the BS is given by
\begin{eqnarray}
\bm{y}=\sum_{i\in\mathcal{S}}\bm{h}_i{b}_is_i+\bm{n},
\end{eqnarray}
where $b_i\in\mathbb{C}$ is the transmitter scalar, $\bm{h}_i\in\mathbb{C}^N$ is the channel vector between device $i$ and the BS, and $\bm{n}\sim\mathcal{CN}(\bm{0}, \sigma^2\bm{I})$ is the noise vector. The transmit power constraint at device $i$ is given by
\begin{equation}\label{eq:transmit_power}
     \mathbb{E}(|b_is_i|^2)=|{b}_i|^2\le P_0
 \end{equation}
with $P_0>0$ as the maximum transmit power. The estimated value before post-processing at the BS is given as
\begin{eqnarray}
\hat{g}=\frac{1}{\sqrt{\eta}}{\bm{m}}^{\sf{H}}\bm{y}=\frac{1}{\sqrt{\eta}}{\bm{m}}^{\sf{H}}\sum_{i\in\mathcal{S}}\bm{h}_i{b}_is_i+\frac{\bm{m}^{\sf{H}}\bm{n}}{\sqrt{\eta}},
\end{eqnarray} 
where $\bm{m}\in\mathbb{C}^N$ is the receiver beamforming vector and $\eta$ is a normalizing factor. Each element of the target vector can thus be obtained as $\hat{z}=\psi(\hat{g})$ at the BS.

The distortion of $\hat{g}$ with respect to the target value $g$ given in equation (\ref{eq:target_function}), which quantifies the over-the-air computation performance for global model aggregation in the FedAvg algorithm, is measured by the mean-squared-error (MSE) defined as
\begin{align}
{\sf{MSE}}&(\hat{g}, g)=\mathbb{E}\left(|\hat{g}-g|^2\right)\nonumber \\&=\sum_{i\in\mathcal{S}}\Big|\bs{m}^{\sf{H}}\bs{h}_ib_i/\sqrt{\eta}-\phi_i\Big|^2+\sigma^2\|\bs{m}\|^2/\eta.
\end{align}
Motivated by \cite{chen2018uniform}, we have the following proposition for transmitter beamformers:
\begin{proposition}\label{prop:tx_beam}
    Given a receiver beamforming vector $\bm{m}$, the MSE is minimized by the following zero-forcing transmitter:
    \begin{equation}\label{eq:trans_scalar}
    b_i=\sqrt{\eta}\phi_i\frac{(\bm{m}^{\sf{H}}\bm{h}_i)^{\sf{H}}}{\|\bm{m}^{\sf{H}}\bm{h}_i\|^2}.
    \end{equation}
\end{proposition}
\begin{proof}
    See Appendix \ref{appd:tx_beam}.
\end{proof}

Due to the transmit power constraint (\ref{eq:transmit_power}) for transmit scalar $b_i$ given in (\ref{eq:trans_scalar}), we have
\begin{equation}
    \eta= \min_{i\in\mathcal{S}}\frac{P_0\|\bm{m}^{\sf{H}}\bm{h}_i\|^2}{\phi_i^2}.
\end{equation}
The MSE is thus given as
\begin{eqnarray}
{\sf{MSE}}(\hat{g}, g;\mathcal{S}, \bm{m})=\frac{\|\bm{m}\|^2\sigma^2}{\eta}=\frac{\sigma^2}{P_0}\max_{i\in\mathcal{S}}\phi_i^2 \frac{\|\bm{m}\|^2}{\|\bm{m}^{\sf{H}}\bm{h}_i\|^2}.
\end{eqnarray}



\subsection{Problem Formulation}
As discussed in Section {\ref{fedlearning}}, the number of selected devices shall be maximized to improve the learning performance for distributed federated learning. In addition, the aggregation error through over-the-air computation is supposed to be reduced to avoid the notable drop of model prediction accuracy. In this paper, we propose to find the maximum selected devices while guaranteeing the MSE requirement for over-the-air computation. It is formulated as the following mixed combinatorial optimization problem
\begin{eqnarray}
\mathop {\textrm{maximize}}_{\mathcal{S}, \bm{m}\in\mathbb{C}^N}&&|\mathcal{S}|\nonumber\\
\textrm{subject to}&&\left(\max_{i\in\mathcal{S}} \phi_i^2
\frac{\|\bm{m}\|^2}{\|\bm{m}^{\sf{H}}\bm{h}_i\|^2}\right)\le \gamma, \label{eq:form1}
\end{eqnarray}
where $\gamma>0$ is the MSE requirement for model aggregation. However, the mixed combinatorial optimization problem (\ref{eq:form1}) is highly intractable due to the combinatorial objective function $|\mathcal{S}|$ and the nonconvex MSE constraint with coupled combinatorial variable $\mathcal{S}$ and continuous variable $\bm{m}$. To address the nonconvexity of MSE function, \cite{chen2018uniform} finds the connections between the nonconvex MSE constraint (\ref{eq:form1}) and the nonconvex quadratic constraints for efficient algorithm designing. Enlightened by this observation, we will show that problem (\ref{eq:form1}) can be equivalently solved by maximizing the number of feasible nonconvex quadratic constraints. Specifically, to support efficient algorithms design, we shall propose a sparse representation  approach to find the maximum number of involved devices, followed by reformulating the nonconvex quadratic constraints as affine constraints with an additional rank-one  constraint by the matrix lifting technique. 

\section{Sparse and Low-Rank Optimization for On-Device Distributed Federated Learning}\label{sec:sparse_and_low_rank}
In this section, we propose a sparse and low-rank optimization modeling approach for on-device distributed federated learning with device selection. 

\subsection{Sparse and Low-Rank Optimization}
To support efficient algorithms design, we first rewrite problem (\ref{eq:form1})  as the mixed combinatorial optimization problem with nonconvex quadratic constraints as presented in Proposition \ref{prop:reform1}.
\begin{proposition}\label{prop:reform1}
        Problem \eqref{eq:form1} is equivalent to the following mixed combinatorial optimization problem:
\begin{eqnarray}
\mathop{\textrm{maximize}}_{\mathcal{S}, \bm{m}\in\mathbb{C}^N}&& |\mathcal{S}| \nonumber\\
\textrm{subject to}&&  \|\bm{m}\|^2-\gamma_i{\|\bm{m}^{\sf{H}}\bm{h}_i\|^2} \le 0, i\in\mathcal{S}, \nonumber\\
&& \|\bm{m}\|^2\geq 1, \label{eq:card1}
\end{eqnarray}
where $\gamma_i = \gamma/\phi_i^2$. That is, our target becomes maximizing the number of feasible MSE constraints $\|\bm{m}\|^2-\gamma_i{\|\bm{m}^{\sf{H}}\bm{h}_i\|^2} \le 0$ under the regularity condition $\|\bm{m}\|^2\geq 1$.
\end{proposition}

\begin{proof}
        Problem \eqref{eq:form1} can be reformulated as
        \begin{eqnarray}
        \mathop {\textrm{maximize}}_{\mathcal{S}, \bm{m}\in\mathbb{C}^N}&&|\mathcal{S}|\nonumber\\
        \textrm{subject to}&&F_i(\bm{m})=\|\bm{m}\|^2-\gamma_i{\|\bm{m}^{\sf{H}}\bm{h}_i\|^2} \le 0, i\in\mathcal{S} \nonumber\\
&& \bm{m}\ne \bm{0},
        \end{eqnarray}
which is further equivalently rewritten as
\begin{eqnarray}
        \mathop {\textrm{maximize}}_{\mathcal{S}, \bm{m}\in\mathbb{C}^N}&&|\mathcal{S}|\nonumber\\
        \textrm{subject to}&&F_i(\bm{m})/\tau=\|\bm{m}\|^2/\tau-\gamma_i{\|\bm{m}^{\sf{H}}\bm{h}_i\|^2}/\tau \le 0, i\in\mathcal{S} \nonumber\\
&& \|\bm{m}\|^2 \geq \tau, \tau>0. \label{prob:form_equivalent}
        \end{eqnarray}
Then by introducing variable $\tilde{\bm{m}}=\bm{m}/\sqrt{\tau}$, problem (\ref{prob:form_equivalent}) can be reformulated as
\begin{eqnarray}
\mathop{\textrm{maximize}}_{\mathcal{S}, \tilde{\bm{m}}\in\mathbb{C}^N}&& |\mathcal{S}| \nonumber\\
\textrm{subject to}&&  F_i(\tilde{\bm{m}})=\|\tilde{\bm{m}}\|^2-\gamma_i{\|\tilde{\bm{m}}^{\sf{H}}\bm{h}_i\|^2} \le 0,\quad i\in\mathcal{S}, \nonumber\\
&& \|\tilde{\bm{m}}\|^2\geq 1.
\end{eqnarray}
Therefore, problem (\ref{eq:form1}) is equivalent to problem (\ref{eq:card1}), where the regularity condition $\|\bm{m}\|^2\geq 1$ serves the purpose of avoiding the singularity (i.e., $\bm{m}=0$).
\end{proof}

To maximize the number of feasible MSE constraints in problem \eqref{eq:card1}, we can minimize the number of nonzero $x_k$'s \cite{shi2016smoothed}, i.e.,
\begin{eqnarray}
\mathop{\textrm{minimize}}_{\bm{x}\in\mathbb{R}_+^M, \bm{m}\in\mathbb{C}^N}&&\|\bm{x}\|_0 \nonumber\\
\textrm{subject to}&&\|\bm{m}\|^2-\gamma_i{\|\bm{m}^{\sf{H}}\bm{h}_i\|^2}\le x_i, \forall i,  \nonumber\\
&& \|\bm{m}\|^2\geq 1. \label{eq:form3}
\end{eqnarray}
The sparsity structure of $\bs{x}$ indicates the feasibility of each mobile device. If $x_i=0$, the $i$-th mobile device can be selected while satisfying the MSE requirement. 

However, both the MSE constraints and the regularity condition in problem (\ref{eq:form3}) are nonconvex quadratic constraints. To addressed this nonconvexity issue, a natural way is adopting the matrix lifting technique \cite{sidiropoulos2006transmit}. Specifically, by lifting vector $\bs{m}$ as the positive semidefinite (PSD) matrix $\bs{M}=\bs{m}\bs{m}^{\sf{H}}$ with $\textrm{rank}(\bs{M})=1$, problem (\ref{eq:form3}) can be reformulated as the following sparse and low-rank optimization problem
\begin{eqnarray}\label{eq:PSD}
\hspace{-1em}\mathscr{P}:\hspace{-1em}\mathop{\textrm{minimize}}_{\bm{x}\in\mathbb{R}_+^M, \bm{M}\in\mathbb{C}^{N\times N}}&&\|\bm{x}\|_0 \nonumber\\
\textrm{subject to~~~~}&&\textrm{Tr}(\bm{M})-\gamma_i\bm{h}_i^{\sf{H}}\bm{M}\bm{h}_i\le x_i, \forall i, \nonumber\\
&&\bm{M}\succeq\bs{0}, \textrm{Tr}(\bs{M})\geq 1, \nonumber \\
&& \textrm{rank}(\bs{M})=1.
\end{eqnarray}
Although problem $\mathscr{P}$ is still nonconvex, we shall demonstrate its algorithmic advantages by developing efficient algorithms.

\subsection{Problem Analysis}
Problem $\mathscr{P}$ is a nonconvex optimization problem with sparse objective function and low-rank constraint. Sparse optimization and low-rank optimization have attracted much attention in machine learning, signal processing, high-dimensional statistics, as well as wireless communication \cite{Yuanming_TWC2014,tropp2010computational, Romberg_JSTSP16, shi2016low, Yuanming_ComMage18}. Although the sparse function and the low rank function are both nonconvex and computationally difficult, significant progress has been achieved for taming the nonconvexity via developing efficient and provable algorithms by exploiting various problem structures.

\subsubsection{Sparse Optimization}
$\ell_1$-norm is a natural convex surrogate for the nonconvex sparse function, i.e., $\ell_0$-norm. The resulting problem is known as the sum-of-infeasibilities in the literature of optimization \cite{boyd2004convex}. Another known approach for enhancing sparsity is the smoothed $\ell_p$-minimization \cite{shi2016smoothed} by finding a tight approximation for the nonconvex $\ell_0$-norm, followed by the iteratively reweighted $\ell_2$-minimization algorithm. However, the smoothing parameters should be chosen carefully  since the convergence behavior of iterative reweighted algorithms may be sensitive to them \cite{chartrand2008iteratively, Yuanming_cvxsmooth18}.

\subsubsection{Low-Rank Optimization}
Simply dropping the rank-one constraint in problem $\mathscr{P}$ yields the semidefinite relaxation (SDR) technique \cite{luo2007approximation}.  The SDR technique is widely used as an effective approach to find approximate solutions for the nonconvex quadratic constrained quadratic programs. If the solution fails to be rank-one, we can obtain a rank-one approximate solution through the Gaussian randomization method \cite{luo2007approximation}. However, when the number of antennas $N$ increases, its performance deteriorates since the probability of returning rank-one solutions is low \cite{chen2018uniform,chen2017admm}.

To address the limitations of the existing works, in this paper, we shall propose a unified difference-of-convex-functions (DC) programming approach to solve the sparse and low-rank optimization problem $\mathscr{P}$. This approach is able to enhance the sparsity in the objective as well as accurately detect the infeasibility in the nonconvex quadratic constraints, yielding considerably improvements compared with state-of-the-art algorithms. Specifically,
\begin{itemize}
    \item We will develop a parameter-free DC approach to enhance sparsity, thereby maximizing the number of selected devices.
    \item Instead of dropping the rank-one constraint directly, we will propose a novel DC approach to guarantee the exact rank-one constraint.
\end{itemize}
Note that the proposed DC approach has the capability of guarantee the feasibility of rank-one constraint, which is critical for accurately detecting the feasibility of the nonconvex quadratic constraints in the procedure of device selection.

\section{DC  Representation for the Sparse and Low-Rank Functions}\label{sec:algorithm}
In this section, we shall propose a unified DC representation framework to solve the sparse and low-rank optimization problem $\mathscr{P}$ for federated learning with device selection. Specifically, the sparsity is induced by a novel DC representation for the $\ell_0$-norm. The sparsity structure provides a guideline for device selection. We then solve a sequence of feasibility detection problems with nonconvex quadratic constraints to find maximum selected devices. In particular, we present a novel DC representation for the rank function in the lifting problem to satisfy the rank-one constraint, which is capable of accurately detecting the feasibility of nonconvex quadratic programs during device selection procedure.

\subsection{DC Representation for Sparse Function}
Before introducing the DC representation for the $\ell_0$-norm, we first give the definition of Ky Fan $k$-norm.
\begin{definition}{Ky Fan $k$-norm \cite{fan1951maximum}:}
        The Ky Fan $k$-norm of vector $\bs{x}\in\mathbb{C}^{M}$ is a convex function of $\bs{x}$ and is given by the sum of largest-$k$ absolute values, i.e.,
        \begin{equation}
                |\!|\!|\bs{x}|\!|\!|_k = \sum_{i=1}^M |x_{\pi(i)}|,
        \end{equation}
        where $\pi$ is a permutation of $\{1,\cdots,M\}$ and $|x_{\pi(1)}|\geq \cdots \geq |x_{\pi(M)}|$.
\end{definition}
If the $\ell_0$-norm is less than $k$, its $\ell_1$-norm is equal to its Ky Fan $k$-norm. Based on this fact, the $\ell_0$-norm can be represented by the difference between $\ell_1$-norm and Ky Fan $k$-norm \cite{gotoh2017dc}:
\begin{equation}
        \|\bs{x}\|_0 = \min\{k:\|\bs{x}\|_1-|\!|\!|\bs{x}|\!|\!|_k=0, 0\leq k\leq M\}.
\end{equation}

\subsection{DC Representation for Low-Rank Constraint}
For the positive semidefinite (PSD) matrix $\bs{M}\in\mathbb{C}^{N\times N}$, the rank-one constraint can be equivalently rewritten as
\begin{equation}
        \sigma_i(\bs{M})=0,\forall i=2,\cdots,N,
\end{equation}
where $\sigma_i(\bs{M})$ is the $i$-th largest singular value of matrix $\bs{M}$. Note that the trace norm and spectral norm are given by
\begin{equation}
        \textrm{Tr}(\bs{M})=\sum_{i=1}^{N}\sigma_i(\bs{M}) ~\text{and}~\|\bs{M}\|_2=\sigma_1(\bs{M}),
\end{equation}
respectively. Therefore, we have the following proposition:
\begin{proposition}
        For PSD matrix $\bs{M}$ and $\textrm{Tr}(\bs{M})\geq 1$, we have
        \begin{equation}\label{eq:DC_lowrank}
        \textrm{rank}(\bs{M})= 1 \Leftrightarrow \textrm{Tr}(\bs{M})-\|\bs{M}\|_2 = 0.
\end{equation}
\end{proposition}
\begin{proof}
        If the rank of PSD matrix $\bs{M}$ is one, the trace norm is equal to the spectral norm as $\sigma_i(\bs{M})=0$ for all $i\geq 2$. The equation $\textrm{Tr}(\bs{M})-\|\bs{M}\|_2 = 0$ implies that $\sigma_i(\bs{M})=0$ for all $i\geq 2$, i.e., $\textrm{rank}(\bs{M})\leq 1$. And we have $\sigma_1(\bs{M})>0$ from $\textrm{Tr}(\bs{M})\geq 1$. Therefore, $\textrm{rank}(\bs{M})= 1$ holds if $\textrm{Tr}(\bs{M})-\|\bs{M}\|_2 = 0$.
\end{proof}

\subsection{A Unified DC Representation Framework}
The main idea of our proposed DC representation framework is to induce the sparsity of $\bs{x}$ in the first step, which will provide guidelines for determining the priority of selecting devices. Then we shall solve a series of feasibility detection problems to find maximum selected devices such that the MSE requirement is satisfied. This two-step framework is illustrated in Fig. \ref{fig:two_step}. And each step will be accomplished by solving a DC program.
\begin{figure}[h]
    \centering
    \includegraphics[width=0.9\columnwidth]{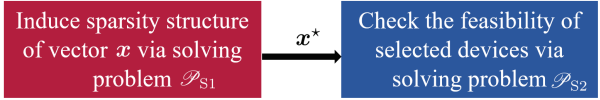}
    \caption{A two-step framework for device selection.}
    \label{fig:two_step}
\end{figure}
\subsubsection{Step I: Sparsity Inducing}
In the first step, we solve the following DC program for problem $\mathscr{P}$:
\begin{eqnarray}
\mathscr{P}_{\text{S1}}:\mathop {\textrm{minimize}}_{\bm{x},\bm{M}}&&\|\bm{x}\|_1-|\!|\!|\bs{x}|\!|\!|_{k}+ \textrm{Tr}(\bs{M})-\|\bs{M}\|_2 \nonumber\\
\textrm{subject to}&&\textrm{Tr}(\bm{M})-\gamma_i\bm{h}_i^{\sf{H}}\bm{M}\bm{h}_i\le x_i, \forall i=1,\cdots,M \nonumber\\
&&\bm{M}\succeq\bs{0},\quad \textrm{Tr}(\bs{M})\geq 1, \bm{x}\succeq\bs{0}.
\end{eqnarray}
By sequentially solving problem $\mathscr{P}_{\text{S1}}$, we can obtain the sparse vector $\bs{x}^{\star}$ such that the objective value achieves zero through increasing $k$ from $0$ to $M$. Note that the rank one constraint of matrix $\bm{M}$ shall be satisfied when the objective value equals zero with $\textrm{Tr}(\bs{M})-\|\bs{M}\|_2 = 0$. 

\subsubsection{Step II: Feasibility Detection}
The solution $\bs{x}$ obtained in the first step characterizes the gap between the MSE requirement and the achievable MSE for each device. Therefore, in the second step, we propose to select device $k$ with higher priority if $x_k$ is small. The elements of $\bs{x}$ can be arranged in descending order $x_{\pi(1)}\geq \cdots \geq x_{\pi(M)}$. We will find the minimum $k$ by increasing $k$ from $1$ to $M$ such that selecting all devices in $\mathcal{S}^{[k]}$ is feasible, where the set $\mathcal{S}^{[k]}$ is chosen as $\{\pi(k),\pi(k+1),\cdots,\pi(M)\}$. 

In detail, if all devices in $\mathcal{S}^{[k]}$ can be selected, the following optimization problem
\begin{eqnarray}
\mathop {\textrm{find}} &&{\bm{m}} \nonumber\\
\textrm{subject to}&& \|\bm{m}\|^2-\gamma_i\|\bm{m}^{\sf{H}}\bm{h}_i\|^2\le 0, \forall i\in\mathcal{S}^{[k]} \nonumber\\
&&\|\bm{m}\|^2\geq 1.\label{prob:feasibility_detection}
\end{eqnarray}
should be feasible. It can be equivalently reformulated as
\begin{eqnarray}
\mathop {\textrm{find}} &&{\bm{M}} \nonumber\\
\textrm{subject to}&&\textrm{Tr}(\bm{M})-\gamma_i\bm{h}_i^{\sf{H}}\bm{M}\bm{h}_i\le 0, \forall i\in\mathcal{S}^{[k]} \nonumber\\
&&\bm{M}\succeq\bs{0}, \textrm{Tr}(\bs{M})\geq 1, \textrm{rank}(\bs{M})=1 \label{prob:feasibility_detection_lift}
\end{eqnarray}
using the matrix lifting technique. To guarantee the feasibility of the fixed-rank constraint for accurately detecting the feasibility of MSE constraints, we propose the following DC approach by minimizing the difference between trace norm and spectral norm:
\begin{eqnarray}
\mathscr{P}_{\text{S2}}:\mathop {\textrm{minimize}}_{\bm{M}}&&\textrm{Tr}(\bs{M}) -\|\bs{M}\|_2 \nonumber\\
\textrm{subject to}&&\textrm{Tr}(\bm{M})-\gamma_i\bm{h}_i^{\sf{H}}\bm{M}\bm{h}_i\le 0, \forall i\in\mathcal{S}^{[k]} \nonumber\\
&&\bm{M}\succeq\bs{0},\quad \textrm{Tr}(\bs{M})\geq 1.
\end{eqnarray}

That is, when the objective value of problem $\mathscr{P}_{\text{S2}}$ equals zero given set $\mathcal{S}^{[k]}$, we conclude that all devices in $\mathcal{S}^{[k]}$ are selected while satisfying the MSE requirement, i.e., problem (\ref{prob:feasibility_detection}) is feasible for $\mathcal{S}^{[k]}$. Note that the solution $\bs{M}^*$ shall be an exact rank-one matrix and a feasible receiver beamforming vector $\bs{m}$ can be obtained through Cholesky decomposition $\bs{M}^*=\bs{m}\bs{m}^{\sf{H}}$.

The proposed DC representation framework for solving the sparse and low-rank optimization problem in federated learning is presented in Algorithm \ref{algorithm:device_selection}. Since the DC program is still nonconvex, in next section, we will develop the DC Algorithm (DC) \cite{tao1997convex} for the DC optimization problem $\mathscr{P}_{\text{S1}}$ and problem $\mathscr{P}_{\text{S2}}$. We further contribute by establishing the convergence rate of DC algorithm. Due to the superiority of the presented DC representation (\ref{eq:DC_lowrank}) for rank-one constraint, our proposed DC approach for accurate feasibility detection considerably outperforms the SDR approach \cite{luo2007approximation} by simply dropping the rank-one constraint, which will be demonstrated through numerical experiments in Section V.
\SetNlSty{textbf}{}{:}
\IncMargin{1em}
\begin{algorithm}[h]
\textbf{Step 1:} sparsity inducing \\
$k\leftarrow 0$ \\
\While{objective value of $\mathscr{P}_{\text{S1}}$ is not zero}{
Obtain solution $\bs{x}$ by solving the DC program $\mathscr{P}_{\text{S1}}$ \\
$k \leftarrow k+1$
         }
\vspace{1em}
\textbf{Step 2:} feasibility detection\\
 Order $\bm{x}$ in descending order as $x_{\pi(1)}\geq \cdots \geq x_{\pi(M)}$\\
$k\leftarrow 1$ \\
\While{objective value of $\mathscr{P}_{\text{S2}}$ is not zero}{
$\mathcal{S}^{[k]} \leftarrow  \{\pi(k),\pi(k+1),\cdots,\pi(M)\}$ \\
Obtain solution $\bs{M}$ by solving the DC program $\mathscr{P}_{\text{S2}}$\\
$k \leftarrow k+1$
}
\textbf{Output:} $\bs{m}$ through Cholesky decomposition $\bs{M}=\bs{m}\bs{m}^{\sf{H}}$, and the set of selected devices $\mathcal{S}^{[k]}=\{\pi(k),\pi(k+1),\cdots,\pi(M)\}$
 \caption{DC Representation Framework for Solving Problem $\mathscr{P}$ in Federated Learning with Device Selection}
 \label{algorithm:device_selection}
\end{algorithm}

\section{DC Algorithm for DC Program with Convergence Guarantees}
In this section, the DC Algorithm  will be developed by successively solving the convex relaxation of primal problem and dual problem of DC program. To further establish the convergence results, we add quadratic terms in convex functions while their difference (i.e., the objective value) remains unchanged. With this technique, we represent the DC objective function as the difference of strongly convex functions, which allows us establish the convergence rate of the DC algorithm.

\subsection{Difference-of-Strongly-Convex-Functions Representation}
The DC formulations $\mathscr{P}_{\text{S1}}$ and $\mathscr{P}_{\text{S2}}$ for sparse and low-rank optimization are nonconvex programs with DC objective functions and convex constraints. Although DC functions are nonconvex, they have good problem structures and the DC Algorithm can be developed based on the principles provided in \cite{tao1997convex}. In order to establish the convergence result of the DC algorithm, we will represent the DC objective function as the difference of strongly convex functions.

Specifically, we can equivalently rewrite problem $\mathscr{P}_{\text{S1}}$ as 
\begin{equation}\label{prob:S1_unconstrained}
    \mathop{\textrm{minimize}}_{\bs{x},\bs{M}}~ f_1=\|\bm{x}\|_1-|\!|\!|\bs{x}|\!|\!|_{k}+ \textrm{Tr}(\bs{M})-\|\bs{M}\|_2+I_{\mathcal{C}_1}(\bs{x},\bs{M}),
\end{equation}
and problem $\mathscr{P}_{\text{S2}}$ as
\begin{equation}\label{prob:S2_unconstrained}
    \mathop{\textrm{minimize}}_{\bs{M}}~ f_2=\textrm{Tr}(\bs{M})-\|\bs{M}\|_2+I_{\mathcal{C}_2}(\bs{M}),
\end{equation}
respectively. Here $\mathcal{C}_1,\mathcal{C}_2$ are positive semidefinite cones that integrates the constraints of problem $\mathscr{P}_{\text{S1}}$ and problem $\mathscr{P}_{\text{S2}}$, and the indicator function is defined as
\begin{equation}
    I_{\mathcal{C}_1}(\bs{x},\bs{M}) = \left\{\begin{aligned}
        & 0, && (\bs{x},\bs{M})\in\mathcal{C}_1 \\
        & +\infty, && \text{otherwise}
    \end{aligned}\right..
\end{equation}

In order to establish the convergence result of the DC algorithm, we rewrite the DC functions $f_1,f_2$ as the difference of \textit{strongly} convex functions, i.e., $f_1=g_1-h_1$ and $f_2 = g_2-h_2$, where
\begin{align}
        g_1&=\|\bm{x}\|_1+ \textrm{Tr}(\bs{M})+I_{\mathcal{C}_1}(\bs{x},\bs{M})+\frac{\alpha}{2}(\|\bs{x}\|_F^2+\|\bs{M}\|_F^2),\\ h_1 &= |\!|\!|\bs{x}|\!|\!|_{k}+\|\bs{M}\|_2+\frac{\alpha}{2}(\|\bs{x}\|_F^2+\|\bs{M}\|_F^2),\\
        g_2 &=\textrm{Tr}(\bs{M})+I_{\mathcal{C}_2}(\bs{M})+\frac{\alpha}{2}\|\bs{M}\|_F^2, \\    h_2&=\|\bs{M}\|_2+\frac{\alpha}{2}\|\bs{M}\|_F^2.
\end{align}
By adding quadratic terms, $g_1,g_2,h_1,h_2$ are all $\alpha$-strongly convex functions. Then problem (\ref{prob:S1_unconstrained}) and problem (\ref{prob:S2_unconstrained}) admit the uniform structure of minimizing the difference of two strongly convex functions
\begin{eqnarray}\label{prob:primal}
\mathop {\textrm{minimize}}_{\bm{X}\in\mathbb{C}^{m\times n}}&&f(\bs{X})=g(\bs{X})-h(\bs{X}).
\end{eqnarray}
For complex domain $\bs{X}$, we shall apply Wirtinger calculus \cite{bouboulis2012adaptive} for algorithm design. The DC algorithm is given by constructing sequences of candidates to primal solutions and dual solutions. Since  the primal problem (\ref{prob:primal}) and its dual problem are still nonconvex, convex relaxation is further needed.

\subsection{DC Algorithm for Sparse and Low-Rank Optimization}
According to the Fenchel's duality \cite{rockafellar2015convex}, the dual problem of problem (\ref{prob:primal}) is given by
\begin{eqnarray}\label{prob:dual}
\mathop {\textrm{minimize}}_{\bm{Y}\in\mathbb{C}^{m\times n}}&&h^*(\bs{Y})-g^*(\bs{Y}),
\end{eqnarray}
where $g^*$ and $h^*$ are the conjugate functions of $g$ and $h$, respectively. The conjugate function is defined as
\begin{equation}
    g^*(\bs{Y})=\sup_{\bs{X}\in\mathbb{C}^{m\times n}}\langle \bs{X},\bs{Y} \rangle-g(\bs{X}),
\end{equation}
where $\langle \bs{X},\bs{Y} \rangle=\rm{Real}\big(\textrm{Tr}(\bs{X}^{H}\bs{Y})\big)$ defines the inner product of two matrices \cite{bouboulis2012adaptive}. The $t$-th iteration of the simplified DC algorithm is to solve the convex approximation of primal problem and dual problem by linearizing the concave part: 
\begin{align}
    &\bs{Y}^{[t]}=\arg\inf_{\bs{Y}\in\mathcal{Y}}~h^*(\bs{Y})-[g^*(\bs{Y}^{[t-1]})+\langle \bs{Y}-\bs{Y}^{[t-1]}, \bs{X}^{[t]}\rangle],\label{dc:iter1}\\
    &\bs{X}^{[t+1]}=\arg\inf_{\bs{X}\in\mathcal{X}}~g(\bs{X})-[h(\bs{X}^{[t]})+\langle \bs{X}-\bs{X}^{[t]}, \bs{Y}^{[t]}\rangle]. \label{dc:iter2}
\end{align}
According to the Fenchel biconjugation theorem \cite{rockafellar2015convex}, equation (\ref{dc:iter1}) can be rewritten as
\begin{equation}
    \bs{Y}^{[t]}\in\partial_{\bs{X}^{[t]}} h,
\end{equation}
$\partial_{\bs{X}^{[t]}} h$ is the subgradient of $h$ with respect to $\bs{X}$ at $\bs{X}^{[t]}$. 

Therefore, iterations $\bs{x}^{[t]},\bs{M}^{[t]}$ of the DC algorithm for problem $\mathscr{P}_{\text{S1}}$ are constructed as the solution to the following convex optimization problem
\begin{eqnarray}
\mathop {\textrm{minimize}}_{\bm{x},\bm{M}}&&g_1-\langle\partial_{\bs{x}^{[t-1]}} h_1, \bs{x}\rangle- \langle\partial_{\bs{M}^{[t-1]}} h_1, \bs{M}\rangle\nonumber\\
\textrm{subject to}&&\textrm{Tr}(\bm{M})-\gamma_i\bm{h}_i^{\sf{H}}\bm{M}\bm{h}_i\le x_i, \forall i=1,\cdots,M, \nonumber\\
&&\bm{M}\succeq\bs{0},\quad \textrm{Tr}(\bs{M})\geq 1, \bm{x}\succeq\bs{0} \label{eqalg:S1}.
\end{eqnarray}
The iteration $\bs{M}^{[t]}$ for problem $\mathscr{P}_{\text{S2}}$ is given by the solution to the following optimization problem
\begin{eqnarray}
\mathop {\textrm{minimize}}_{\bm{M}}&&g_2-\langle\partial_{\bs{M}^{[t-1]}} h_2, \bs{M}\rangle\nonumber\\
\textrm{subject to}&&\textrm{Tr}(\bm{M})-\gamma_i\bm{h}_i^{\sf{H}}\bm{M}\bm{h}_i\le 0, \forall i\in\mathcal{S}^{[k]}, \nonumber\\
&&\bm{M}\succeq\bs{0},\quad \textrm{Tr}(\bs{M})\geq 1 \label{eqalg:S2}.
\end{eqnarray}
The subgradient of $h_1$ and $h_2$ are given by
\begin{align}
        \partial_{\bs{x}}h_1 &= \partial |\!|\!|\bs{x}|\!|\!|_{k}+\alpha\bs{x}, \\ \partial_{\bs{M}}h_1 &= \partial_{\bs{M}}h_2 =\partial \|\bs{M}\|_2+\alpha\bs{M}.
\end{align}
The subgradient of $|\!|\!|\bs{x}|\!|\!|_{k}$ can be computed by \cite{gotoh2017dc}
\begin{equation}
        i\text{-th entry of } \partial|\!|\!|\bs{x}|\!|\!|_{k}=\left\{\begin{aligned}
                &\textrm{sign}(x_i), && |x_i|\geq |x_{(k)}| \\
                & 0, && |x_i|< |x_{(k)}|
        \end{aligned}\right..
 \end{equation}
The subgradient of $\|\bs{M}\|_2$ is given by the following proposition.
\begin{proposition}
    The subgradient of $\|\bs{M}\|_2$ can be computed as $\bs{v}_1\bs{v}_1^{\sf{H}}$, where $\bs{v}_1\in\mathbb{C}^N$ is the eigenvector of the largest eigenvalue $\sigma_1(\bs{M})$.
\end{proposition}
\begin{proof}
    The subdifferential of orthogonal invariant norm $\|\bs{M}\|_2$ for PSD matrix $\bs{M}$ is given by \cite{watson1992characterization}
    \begin{equation}
         \partial \|\bs{M}\|_2 ={\rm{conv}}\{\bs{V}\textrm{diag}(\bs{d})\bs{V}^{\sf H}: \bs{d} \in \partial \|\bs{\sigma}(\bs{M})\|_{\infty}\},
    \end{equation}
    where ${\rm{conv}}$ denotes the convex hull of a set and $\bs{M}=\bs{V}\bs{\Sigma}\bs{V}^{\sf{H}}$ is the singular value decomposition of $\bs{M}$, and $\bs{\sigma}(\bs{M})=[\sigma_i(\bs{M})]\in\mathbb{C}^N$ is the vector formed by all singular values of $\bs{M}$. Since $\sigma_1(\bs{M})\geq \cdots \geq \sigma_N(\bs{M})\geq 0$, we have
    \begin{equation}
        [1, \underbrace{0, \cdots, 0}_{N-1}]^{\sf{H}}\in \partial\|\bs{\sigma}(\bs{M})\|_{\infty}.
    \end{equation}
    Therefore, one subgradient of $\|\bs{M}\|_2$ is given by $\bs{v}_1\bs{v}_1^{\sf{H}}$.
\end{proof}

\subsection{Convergence Analysis}
The convergence of the presented DC algorithm for problem $\mathscr{P}_{\text{S1}}$ and problem $\mathscr{P}_{\text{S2}}$ is given by the following proposition.
\begin{proposition}\label{prop:convergence}
    The sequence $\{(\bs{M}^{[t]},\bs{x}^{[t]})\}$ generated by iteratively solving problem (\ref{eqalg:S1}) for problem $\mathscr{P}_{\text{S1}}$ has the following properties:
    \begin{enumerate}[(i)]
        \item \label{property1} The sequence $\{(\bs{M}^{[t]},\bs{x}^{[t]})\}$ converges to a critical point of $f_1$ (\ref{prob:S1_unconstrained}) from arbitrary initial point, and the sequence of $\{f_1^{[t]}\}$ is strictly decreasing and convergent.
        \item \label{property2} For any $t=0,1,\cdots$, we have
        \begin{align}
        \text{Avg}\Big(\|\bs{M}^{[t]}-\bs{M}^{[t+1]}\|_F^2\Big)&\leq \frac{f_1^{[0]}-f_1^{\star}}{\alpha(t+1)},\\
        \text{Avg}\Big(\|\bs{x}^{[t]}-\bs{x}^{[t+1]}\|_2^2\Big)&\leq \frac{f_1^{[0]}-f_1^{\star}}{\alpha(t+1)},
    \end{align}
    where $f_1^{\star}$ is the global minimum of $f_1$ and $\text{Avg}\Big(\|\bs{M}^{[t]}-\bs{M}^{[t+1]}\|_F^2\Big)$ denotes the average of the sequence $\{\|\bs{M}^{[i]}-\bs{M}^{[i+1]}\|_F^2\}_{i=0}^{t}$.
    \end{enumerate}   
    Likewise, the sequence $\{(\bs{M}^{[t]}\}$ generated by iteratively solving problem (\ref{eqalg:S2}) for problem $\mathscr{P}_{\text{S2}}$ has the following properties:
    \begin{enumerate}[(i)]\addtocounter{enumi}{2}
        \item \label{property3} The sequence $\{\bs{M}^{[t]}\}$ converges to a critical point of $f_2$ (\ref{prob:S2_unconstrained}) from arbitrary initial point, and the sequence of $\{f_2^{[t]}\}$ is strictly decreasing and convergent.
        \item \label{property4} For any $t=0,1,\cdots$, we have
        \begin{align}
        \text{Avg}\Big(\|\bs{M}^{[t]}-\bs{M}^{[t+1]}\|_F^2\Big)&\leq \frac{f_2^{[0]}-f_2^{\star}}{\alpha(t+1)}.
    \end{align}
    where $f_2^{\star}$ is the global minimum of $f_2$.
    \end{enumerate} 
\end{proposition}
\begin{proof}
    Please refer to Appendix \ref{appd:convergence} for details.
\end{proof}

\section{Simulation Results}
In this section, we conduct numerical experiments to compare the proposed DC method with state-of-the-art approaches for federated learning with device selection.  The channel coefficient vectors $\bs{h}_i$'s between the BS and each mobile device follow the i.i.d. complex normal distribution, i.e., $\bs{h}_i\sim\mathcal{CN}(\bs{0},\bs{I})$. The average transmit signal-to-noise-ratio (SNR) $P_0/\sigma^2$ is chosen as $20$ dB. We assume that all devices have the same number of data points, i.e., $|\mathcal{D}_1|=\cdots=|\mathcal{D}_M|$, for which the pre-processing post-processing pair can be chosen as $\phi_i=1,\psi=1/|\mathcal{S}|$.

\subsection{Probability of Feasibility}
Consider the network with $M=20$ mobile devices and the BS is equipped with $N=6$ antennas. As a critical step for the device selection, the performance of feasibility detection with the proposed DC approach by solving $\mathscr{P}_{\text{S2}}$ shall be compared with the following state-of-the-art approaches:
\begin{itemize}
    \item \textbf{SDR} \cite{luo2007approximation}: Simply dropping the rank-one constraint of problem (\ref{prob:feasibility_detection}) yields the semidefinite relaxation (SDR) approach for the feasibility detection problem.
    \item \textbf{Global Optimization} \cite{lu2017efficient}: In \cite{lu2017efficient}, a global optimization approach is proposed with exponential time complexity in the worst case. We set the relative error tolerance as $\epsilon=10^{-5}$ and take its performance as our benchmark.
\end{itemize}
The results averaged over $100$ times are shown in Fig. \ref{fig:feasibility}, which demonstrates that the proposed DC-based approach outperforms SDR approach significantly and achieves the near-optimal performance compared with  the global optimization approach, and thus yields accurate feasibility detection.
\begin{figure}[h]
        \centering
        \includegraphics[width=0.85\columnwidth]{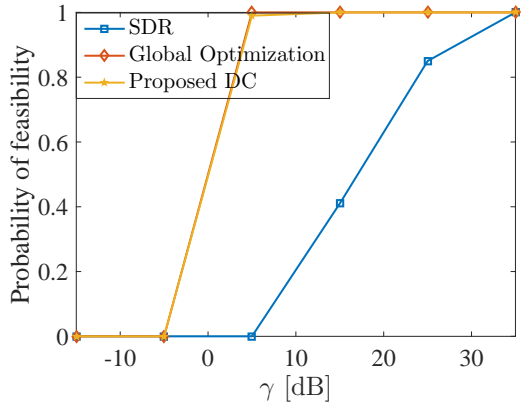}
        \caption{Probability of feasibility with different algorithms.}
        \label{fig:feasibility}
\end{figure}

We then evaluate the performance of the proposed DC approach over the number of antennas. Under different target MSE requirement, the results averaged over $100$ channel realizations are illustrated in Fig. \ref{fig:antenna}. It demonstrates that fast aggregation from mobile devices under a more stringent MSE requirement can be accomplished by increasing the number of antennas at the BS.
\begin{figure}[h]
        \centering
        \includegraphics[width=0.85\columnwidth]{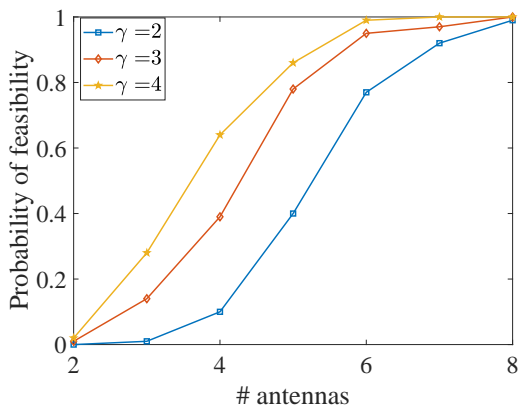}
        \caption{Probability of feasibility over the number of BS antennas with the proposed DC approach.}
        \label{fig:antenna}
\end{figure}

\subsection{Number of Selected Devices over Target MSE}
Consider a network with $20$ mobile devices and a $6$-antenna BS. Under the presented two-step framework and ordering rule in Algorithm \ref{algorithm:device_selection}, we compare the proposed DC Algorithm \ref{algorithm:device_selection} for device selection with the following state-of-the-art approaches:
\begin{itemize}
    \item \textbf{$\ell_1$+SDR} \cite{boyd2004convex} \cite{luo2007approximation}: The $\ell_1$-norm minimization is adopted to induce the sparsity of $\bs{x}$ in Step 1, and the nonconvex quadratic constraints are addressed with the SDR approach in Step 1 and Step 2.  
    \item \textbf{Reweighted $\ell_2$+SDR} \cite{shi2016smoothed}: We take the smoothed $\ell_p$-norm for  sparsity inducing of $\bs{x}$ in Step 1, which is solved by the reweighted $\ell_2$-minimization algorithm. The SDR approach is used to address the nonconvex quadratic program in Step 1 and Step 2.
\end{itemize}
The average results over $100$ channel realizations with different approaches for sparsity inducing and feasibility detection are illustrated in Fig. \ref{fig:selection}. It is demonstrated that the novel sparsity and low-rankness inducing approach via the proposed DC algorithm is able to select more  devices than other state-of-the-art approaches.
 \begin{figure}[h]
  \centering
  \includegraphics[width=0.85\columnwidth]{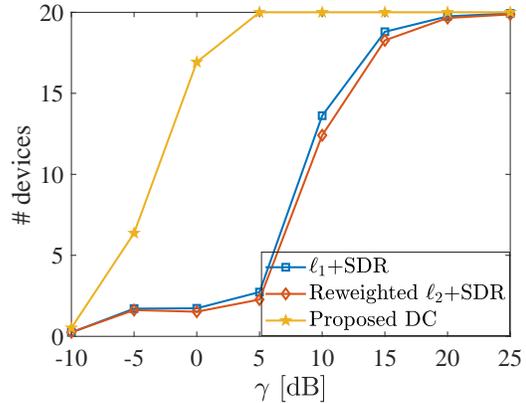}
  \caption{Average number of selected devices with different algorithms.}
  \label{fig:selection}
\end{figure}

\subsection{Performance of Proposed DC Approach for Distributed Federated Learning}
To show the performance of the proposed DC approach for device selection in distributed federated learning, we further train a support vector machine (SVM) classifier on CIFAR-10 dataset \cite{krizhevsky2009learning} with a $6$-antenna BS and $20$ mobile devices. CIFAR-10 is a commonly used dataset of images for classification and contains 10 different classes of objects. The benchmark is chosen as the case where all devices are selected and all local updates are aggregated without aggregation error. We average over $10$ channel realizations and the performances of all algorithms with $\gamma=5$dB are illustrated in Fig. \ref{fig:SVM}. The relative accuracy is defined by the test accuracy over random classification. The simulation results demonstrate that the proposed DC approach achieves lower training loss and higher prediction accuracy as shown in Fig. \ref{fig:loss} and Fig. {\ref{fig:accuracy}}, respectively.

\begin{figure}[h]
        \centering
        \subfloat[Training loss]{\includegraphics[width=0.85\columnwidth]{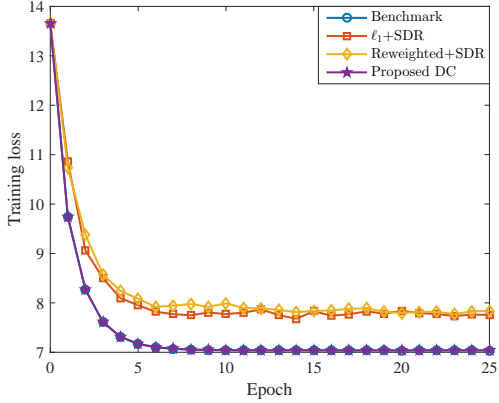}\label{fig:loss}}\hfil
        \subfloat[Relative prediction accuracy]{\includegraphics[width=0.85\columnwidth]{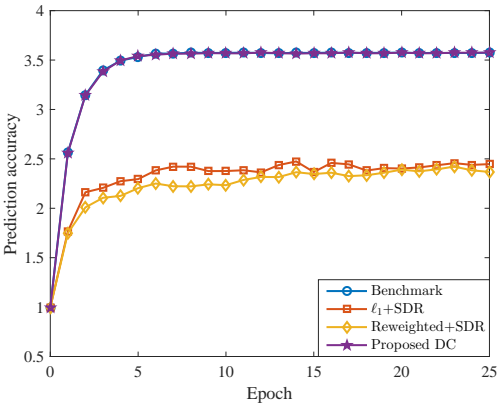}\label{fig:accuracy}}
        \caption{a) Convergence of different device selection algorithms for FedAvg. b) The relationship between communication rounds and test accuracy over random classification of the trained model. Each client updates its local model with stochastic gradient descent algorithm.}
        \label{fig:SVM}
\end{figure}

\section{Conclusion}
In this paper, we proposed a novel fast global model aggregation approach for  federated learning based on the principles of over-the-air computation. To improve the statistical learning performance for on-device distributed training, we developed a novel sparse and low-rank modeling approach to maximize the selected devices  with the MSE requirements for model aggregation. We provided a unified DC representation framework to induce sparsity and low-rankness, which is supported by the convergence guaranteed DC algorithm via successive convex relaxation. Simulation results demonstrated the admirable performance of  the proposed approaches compared with the state-of-the-art algorithms.

There are still some interesting open problems on the fast model aggregation for on-device federated learning including:
\begin{itemize}
        \item This work assumes the perfect channel state information during receiver beamforming. It would be interesting to investigate the impacts of channel uncertainty in model aggregation.
        \item The security issues are also critical for model aggregation, though it is beyond the scope of this paper. It is also interesting to propose a robust approach against the malicious attacks during model aggregation.
        \item The proposed DC approach for feasibility detection has comparable performance with the global optimization approach through numerical experiments. But it remains challenging to characterize its optimality conditions of the DC approach.
\end{itemize}

\appendices
\section{Proof of Proposition \ref{prop:tx_beam}}\label{appd:tx_beam}
The sequence $\{b_i\}$ given by Proposition \ref{prop:tx_beam} has the zero-forcing structure which enforces
\begin{equation}
    \sum_{i\in\mathcal{S}}\Big|\bs{m}^{\sf{H}}\bs{h}_ib_i-\phi_i\Big|^2=0.
\end{equation}
In addition, the MSE satisfies
\begin{align}
    {\sf{MSE}}(\hat{g},g)\geq \sigma^2\|\bs{m}\|^2.
\end{align}
Therefore, the MSE is minimized by the zero-forcing transmitter beamforming vectors $\{b_i\}$'s given in Proposition \ref{prop:tx_beam}.

\section{Proof of Proposition \ref{prop:convergence}}\label{appd:convergence}
Without loss of generality, we shall only present the proof of properties (\ref{property1}) and (\ref{property2}), while properties (\ref{property3}) and (\ref{property4}) can be proved with the same merit. For the sequence $\{(\bs{M}^{[t]},\bs{x}^{[t]})\}$ generated by iteratively solving problem (\ref{eqalg:S1}), we denote the dual variables as $\bs{Y}_M^{[t]}\in \partial_{\bs{M}^{[t]}}h_1 ,\bs{Y}_x^{[t]}\in \partial_{\bs{x}^{[t]}}h_1$. Due to the strong convexity of $h_1$, we have
\begin{align}
    h_1^{[t+1]}-h_1^{[t]}\geq &\langle \Delta_t\bs{M},\bs{Y}_M^{[t]} \rangle +\langle \Delta_{t}\bs{x},\bs{Y}_x^{[t]} \rangle \nonumber \\ &+\frac{\alpha}{2}\big(\|\Delta_t\bs{M}\|_F^2+\|\Delta_t\bs{x}\|_2^2\big), \label{appd:eq1} \\
    \langle \bs{M}^{[t]},\bs{Y}_M^{[t]} \rangle &+\langle \bs{x}^{[t]},\bs{Y}_x^{[t]} \rangle = h_1^{[t]}+{h_1^*}^{[t]}, \label{appd:fenchel_eq1}
\end{align}
where $\Delta_t\bs{M}=\bs{M}^{[t+1]}-\bs{M}^{[t]}$ and $\Delta_t\bs{x}=\bs{x}^{[t+1]}-\bs{x}^{[t]}$. Adding $g_1^{[t+1]}$ at both sides of (\ref{appd:eq1}), we obtain that
\begin{align}
    f_1^{[t+1]}\leq &g_1^{[t+1]}-h_1^{[t]}-\langle \Delta_t\bs{M},\bs{Y}_M^{[t]} \rangle +\langle \Delta_t\bs{x},\bs{Y}_x^{[t]} \rangle \nonumber \\ &-\frac{\alpha}{2}\big(\|\Delta_t\bs{M}\|_F^2+\|\Delta_t\bs{x}\|_2^2\big). \label{appd:eq2}
\end{align}

For the update of primal variable $\bs{M}$ and $\bs{x}$ according to equation (\ref{dc:iter2}), we have $\bs{Y}_{M}^{[t]}\in \partial_{\bs{M}^{[t+1]}} {g_1}, \bs{Y}_{x}^{[t]}\in \partial_{\bs{x}^{[t+1]}} g_1$. This implies that
\begin{align}
    g_1^{[t]}-g_1^{[t+1]}\geq &\langle -\Delta_t\bs{M}, \bs{Y}_M^{[t]} \rangle+\langle -\Delta_t\bs{x}, \bs{Y}_x^{[t]} \rangle \nonumber\\ &+\frac{\alpha}{2}\big(\|\Delta_t\bs{M}\|_F^2+\|\Delta_t\bs{x}\|_2^2\big), \label{appd:eq3} \\
    \langle \bs{M}^{[t+1]},\bs{Y}_M^{[t]} \rangle &+\langle \bs{x}^{[t+1]},\bs{Y}_x^{[t]} \rangle = g_1^{[t+1]}+{g_1^*}^{[t]}. \label{appd:fenchel_eq2}
\end{align}
Similarly, by adding $-h_1^{[t]}$ at both sides of equation (\ref{appd:eq3}), we have
\begin{align}
    f_1^{[t]}\geq &g_1^{[t+1]}-h_1^{[t]}+\langle  -\Delta_t\bs{M}, \bs{Y}_M^{[t]} \rangle+\langle -\Delta_t\bs{x}, \bs{Y}_x^{[t]} \rangle \nonumber\\ &+\frac{\alpha}{2}\big(\|\Delta_t\bs{M}\|_F^2+\|\Delta_t\bs{x}\|_2^2\big). \label{appd:eq4}
\end{align}
From equation (\ref{appd:fenchel_eq1}) and equation (\ref{appd:fenchel_eq2}), we deduce that
\begin{equation}
    g_1^{[t+1]}-h_1^{[t]}+\langle  -\Delta_t\bs{M}, \bs{Y}_M^{[t]} \rangle+ \langle -\Delta_t\bs{x}, \bs{Y}_x^{[t]}\rangle={f_1^*}^{[t]}, \label{appd:eq5}
\end{equation}
where $f_1^* = h_1^*-g_1^*$.
Combining equation (\ref{appd:eq2}), (\ref{appd:eq4}) and (\ref{appd:eq5}), it is derived that
\begin{align}
    f_1^{[t]}&\geq {f_1^*}^{[t]}+\frac{\alpha}{2}\big(\|\Delta_t\bs{M}\|_F^2+\|\Delta_t\bs{x}\|_2^2\big) \\&\geq f_1^{[t+1]}+\alpha\big(\|\Delta_t\bs{M}\|_F^2+\|\Delta_t\bs{x}\|_2^2\big).
\end{align}


Then the sequence $\{f_1^{[t]}\}$ is non-increasing. Since $f_1\geq 0$ always holds, we conclude that the sequence $\{f_1^{[t]}\}$ is strictly decreasing until convergence, i.e.,
\begin{equation}
    \lim_{t\rightarrow \infty} \big(\|\bs{M}^{[t]}-\bs{M}^{[t+1]}\|_F^2+\|\bs{x}^{[t]}-\bs{x}^{[t+1]}\|_2^2\big) = 0.
\end{equation}
For every limit point, $f_1^{[t+1]}=f_1^{[t]}$, we have
\begin{equation}
    \|\bs{M}^{[t]}-\bs{M}^{[t+1]}\|_F^2=0, \|\bs{x}^{[t]}-\bs{x}^{[t+1]}\|_2^2=0,
\end{equation}
and
 \begin{equation}
    f^{[t+1]}={f^*}^{[t]}=f^{[t]}.
\end{equation}
Then it is followed by
 \begin{align}
    {h^*}^{[t]}+h^{[t+1]}&={g}^{[t]}+g^{[t+1]} \nonumber\\
    &=\langle \bs{M}^{[t+1]},\bs{Y}_M^{[t]} \rangle+\langle \bs{x}^{[t+1]},\bs{Y}_x^{[t]} \rangle,
\end{align}
i.e.,
\begin{equation}
    \bs{Y}_M^{[t]}\in\partial_{\bs{M}^{[t+1]}} h_1, \bs{Y}_x^{[t]} \in\partial_{\bs{x}^{[t+1]}} h_1.
\end{equation}
Therefore, $\bs{Y}_M^{[t]}\in \partial_{\bs{M}^{[t+1]}} g_1\cap\partial_{\bs{M}^{[t+1]}} h_1, \bs{Y}_x^{[t]}\in \partial_{\bs{x}^{[t+1]}} g_1\cap\partial_{\bs{x}^{[t+1]}} h_1$. It is concluded that $(\bs{M}^{[t+1]},\bs{x}^{[t+1]})$ is a critical point of $f_1=g_1-h_1$. In addition, since
\begin{align}
    \text{Avg}\Big(\|\bs{M}^{[t]}-&\bs{M}^{[t+1]}\|_F^2+\|\bs{x}^{[t]}-\bs{x}^{[t+1]}\|_2^2\Big) \nonumber \\
    &\leq \sum_{i=0}^{t}\frac{1}{\alpha(t+1)}(f_1^{[i]}-f_1^{[i+1]}) \\
    &\leq \frac{1}{\alpha(t+1)}(f_1^{[0]}-f_1^{[t+1]}) \\
    &\leq \frac{1}{\alpha(t+1)}(f_1^{[0]}-f_1^{\star}),
\end{align}
we conclude that property (\ref{property2}) holds, i.e.,
\begin{align}
    \text{Avg}\Big(\|\bs{M}^{[t]}-\bs{M}^{[t+1]}\|_F^2\Big)&\leq \frac{f_1^{[0]}-f_1^{\star}}{\alpha(t+1)},\\
    \text{Avg}\Big(\|\bs{x}^{[t]}-\bs{x}^{[t+1]}\|_2^2\Big)&\leq \frac{f_1^{[0]}-f_1^{\star}}{\alpha(t+1)}.
\end{align}

\bibliographystyle{ieeetr}
\bibliography{aircomp}

\begin{thebibliography}{10}

\bibitem{lecun2015deep}
Y.~LeCun, Y.~Bengio, and G.~Hinton, ``Deep learning,'' {\em nature}, vol.~521,
  no.~7553, p.~436, 2015.

\bibitem{stoica2017berkeley}
I.~Stoica, D.~Song, R.~A. Popa, D.~Patterson, M.~W. Mahoney, R.~Katz, A.~D.
  Joseph, M.~Jordan, J.~M. Hellerstein, J.~E. Gonzalez, {\em et~al.}, ``A
  berkeley view of systems challenges for {AI},'' {\em arXiv preprint
  arXiv:1712.05855}, 2017.

\bibitem{zhu2018towards}
G.~Zhu, D.~Liu, Y.~Du, C.~You, J.~Zhang, and K.~Huang, ``Towards an intelligent
  edge: Wireless communication meets machine learning,'' {\em arXiv preprint
  arXiv:1809.00343}, 2018.

\bibitem{park2018edgeai}
J.~Park, S.~Samarakoon, M.~Bennis, and M.~Debbah, ``Wireless network
  intelligence at the edge,'' {\em arXiv preprint arXiv:1812.02858}, 2018.

\bibitem{han2015deep}
S.~Han, H.~Mao, and W.~J. Dally, ``Deep compression: Compressing deep neural
  networks with pruning, trained quantization and huffman coding,'' {\em Proc.
  Int. Conf. Learn. Representations (ICLR)}, 2016.

\bibitem{Zhang_SPM18}
Y.~Cheng, D.~Wang, P.~Zhou, and T.~Zhang, ``Model compression and acceleration
  for deep neural networks: The principles, progress, and challenges,'' {\em
  IEEE Signal Process. Mag.}, vol.~35, pp.~126--136, Jan. 2018.

\bibitem{lin2017deep}
Y.~Lin, S.~Han, H.~Mao, Y.~Wang, and W.~J. Dally, ``Deep gradient compression:
  Reducing the communication bandwidth for distributed training,'' {\em Proc.
  Int. Conf. Learn. Representations (ICLR)}, 2018.

\bibitem{wang2018edge}
S.~Wang, T.~Tuor, T.~Salonidis, K.~K. Leung, C.~Makaya, T.~He, and K.~Chan,
  ``When edge meets learning: Adaptive control for resource-constrained
  distributed machine learning,'' in {\em Proc. IEEE Conf. Computer Commun.
  (INFOCOM)}, 2018.

\bibitem{karakus2017straggler}
C.~Karakus, Y.~Sun, S.~Diggavi, and W.~Yin, ``Straggler mitigation in
  distributed optimization through data encoding,'' in {\em Proc. Neural Inf.
  Process. Syst. (NeurIPS)}, pp.~5434--5442, 2017.

\bibitem{mcmahan2017communication}
B.~McMahan, E.~Moore, D.~Ramage, S.~Hampson, and B.~A. y~Arcas,
  ``Communication-efficient learning of deep networks from decentralized
  data,'' in {\em Proc. Int. Conf. Artificial Intell. Stat. (AISTATS)},
  vol.~54, pp.~1273--1282, 2017.

\bibitem{smith_nips2017federated}
V.~Smith, C.-K. Chiang, M.~Sanjabi, and A.~S. Talwalkar, ``Federated multi-task
  learning,'' in {\em Proc. Neural Inf. Process. Syst. (NeurIPS)},
  pp.~4424--4434, 2017.

\bibitem{bonawitz2019towards}
K.~Bonawitz, H.~Eichner, W.~Grieskamp, D.~Huba, A.~Ingerman, V.~Ivanov,
  C.~Kiddon, J.~Konecny, S.~Mazzocchi, H.~B. McMahan, {\em et~al.}, ``Towards
  federated learning at scale: System design,'' {\em arXiv preprint
  arXiv:1902.01046}, 2019.

\bibitem{yang2019federated}
Q.~Yang, Y.~Liu, T.~Chen, and Y.~Tong, ``Federated machine learning: Concept
  and applications,'' {\em ACM Trans. Intell. Syst. Technol.}, vol.~10, no.~2,
  p.~12, 2019.

\bibitem{tran2019federatedi}
N.~H. Tran, W.~Bao, A.~Zomaya, M.~N. Nguyen, and C.~S. Hong, ``Federated
  learning over wireless networks: Optimization model design and analysis,''
  {\em Proc. IEEE Conf. Comput. Commun. (INFOCOM)}, 2019.

\bibitem{zhao2018federated}
Y.~Zhao, M.~Li, L.~Lai, N.~Suda, D.~Civin, and V.~Chandra, ``Federated learning
  with non-iid data,'' {\em arXiv preprint arXiv:1806.00582}, 2018.

\bibitem{wang2017giant}
S.~Wang, F.~Roosta-Khorasani, P.~Xu, and M.~W. Mahoney, ``{GIANT}: Globally
  improved approximate newton method for distributed optimization,'' {\em Proc.
  Neural Inf. Process. Syst. (NeurIPS)}, 2018.

\bibitem{li2018near}
S.~Li, S.~M.~M. Kalan, A.~S. Avestimehr, and M.~Soltanolkotabi, ``Near-optimal
  straggler mitigation for distributed gradient methods,'' in {\em Proc. IEEE
  Int. Parallel Distrib. Process. Symp. Workshops}, pp.~857--866, 2018.

\bibitem{blanchard2017machine}
P.~Blanchard, R.~Guerraoui, J.~Stainer, {\em et~al.}, ``Machine learning with
  adversaries: Byzantine tolerant gradient descent,'' in {\em Proc. Neural Inf.
  Process. Syst. (NeurIPS)}, pp.~119--129, 2017.

\bibitem{chen2017distributed}
Y.~Chen, L.~Su, and J.~Xu, ``Distributed statistical machine learning in
  adversarial settings: Byzantine gradient descent,'' {\em Proc. ACM Meas.
  Anal. Comput. Syst.}, vol.~1, no.~2, p.~44, 2017.

\bibitem{nazer2007computation}
B.~Nazer and M.~Gastpar, ``Computation over multiple-access channels,'' {\em
  IEEE Trans. Inf. Theory}, vol.~53, pp.~3498--3516, Oct. 2007.

\bibitem{Goldenbaum_TSP13harnessing}
M.~Goldenbaum, H.~Boche, and S.~Sta{\'n}czak, ``Harnessing interference for
  analog function computation in wireless sensor networks,'' {\em IEEE Trans.
  Signal Process.}, vol.~61, pp.~4893--4906, Oct. 2013.

\bibitem{Kaibin_IoT2018mimo}
G.~Zhu and K.~Huang, ``{MIMO} over-the-air computation for high-mobility
  multi-modal sensing,'' {\em IEEE Internet Things J.}, 2018.

\bibitem{wang2018cooperative}
J.~Wang and G.~Joshi, ``Cooperative {SGD}: A unified framework for the design
  and analysis of communication-efficient {SGD} algorithms,'' {\em arXiv
  preprint arXiv:1808.07576}, 2018.

\bibitem{zhu2018low}
G.~Zhu, Y.~Wang, and K.~Huang, ``Low-latency broadband analog aggregation for
  federated edge learning,'' {\em arXiv preprint arXiv:1812.11494}, 2018.

\bibitem{luo2007approximation}
Z.-Q. Luo, N.~D. Sidiropoulos, P.~Tseng, and S.~Zhang, ``Approximation bounds
  for quadratic optimization with homogeneous quadratic constraints,'' {\em
  SIAM J. Optim.}, vol.~18, no.~1, pp.~1--28, 2007.

\bibitem{shi2016smoothed}
Y.~Shi, J.~Cheng, J.~Zhang, B.~Bai, W.~Chen, and K.~B. Letaief, ``Smoothed
  ${L}_p$-minimization for green {Cloud-RAN} with user admission control,''
  {\em IEEE J. Sel. Areas Commun.}, vol.~34, pp.~1022--1036, Apr. 2016.

\bibitem{Yuanming_cvxsmooth18}
H.~Wang, F.~Zhang, Q.~Wu, Y.~Hu, and Y.~Shi, ``Nonconvex and nonsmooth sparse
  optimization via adaptively iterative reweighted methods,'' {\em
  arXiv:1810.10167}, 2018.

\bibitem{chen2018uniform}
L.~{Chen}, X.~{Qin}, and G.~{Wei}, ``A uniform-forcing transceiver design for
  over-the-air function computation,'' {\em IEEE Wireless Commun. Lett.},
  vol.~7, pp.~942--945, Dec. 2018.

\bibitem{gotoh2017dc}
J.-y. Gotoh, A.~Takeda, and K.~Tono, ``{DC} formulations and algorithms for
  sparse optimization problems,'' {\em Math. Program.}, vol.~169, pp.~141--176,
  May 2018.

\bibitem{fan1951maximum}
K.~Fan, ``Maximum properties and inequalities for the eigenvalues of completely
  continuous operators,'' {\em Proc. Nat. Academy Sci.}, vol.~37, no.~11,
  pp.~760--766, 1951.

\bibitem{krizhevsky2009learning}
A.~Krizhevsky and G.~Hinton, ``Learning multiple layers of features from tiny
  images,'' tech. rep., University of Toronto, 2009.

\bibitem{keskar2016large}
N.~S. Keskar, D.~Mudigere, J.~Nocedal, M.~Smelyanskiy, and P.~T.~P. Tang, ``On
  large-batch training for deep learning: Generalization gap and sharp
  minima,'' in {\em Proc. Int. Conf. Learn. Representations (ICLR)}, 2017.

\bibitem{sidiropoulos2006transmit}
N.~D. Sidiropoulos, T.~N. Davidson, and Z.-Q. Luo, ``Transmit beamforming for
  physical-layer multicasting,'' {\em IEEE Trans. Signal Process.}, vol.~54,
  pp.~2239--2251, Jun. 2006.

\bibitem{Yuanming_TWC2014}
Y.~Shi, J.~Zhang, and K.~B. Letaief, ``Group sparse beamforming for green
  {Cloud-RAN},'' {\em IEEE Trans. Wireless Commun.}, vol.~13, pp.~2809--2823,
  May 2014.

\bibitem{tropp2010computational}
J.~A. Tropp and S.~J. Wright, ``Computational methods for sparse solution of
  linear inverse problems,'' {\em Proc. IEEE}, vol.~98, pp.~948--958, Jun.
  2010.

\bibitem{Romberg_JSTSP16}
M.~A. Davenport and J.~Romberg, ``An overview of low-rank matrix recovery from
  incomplete observations,'' {\em IEEE J. Sel. Topics Signal Process.},
  vol.~10, pp.~608--622, Jun. 2016.

\bibitem{shi2016low}
Y.~Shi, J.~Zhang, and K.~B. Letaief, ``Low-rank matrix completion for
  topological interference management by {R}iemannian pursuit,'' {\em IEEE
  Trans. Wireless Commun.}, vol.~15, pp.~4703--4717, Jul. 2016.

\bibitem{Yuanming_ComMage18}
Y.~Shi, J.~Zhang, W.~Chen, and K.~B. Letaief, ``Generalized sparse and low-rank
  optimization for ultra-dense networks,'' {\em IEEE Commun. Mag.}, vol.~56,
  pp.~42--48, Jun. 2018.

\bibitem{boyd2004convex}
S.~Boyd and L.~Vandenberghe, {\em Convex optimization}.
\newblock Cambridge Univ. Press, 2004.

\bibitem{chartrand2008iteratively}
R.~Chartrand and W.~Yin, ``Iteratively reweighted algorithms for compressive
  sensing,'' in {\em Proc. IEEE Int. Conf. Acoustics Speech Signal Process.
  (ICASSP)}, pp.~3869--3872, 2008.

\bibitem{chen2017admm}
E.~Chen and M.~Tao, ``{ADMM}-based fast algorithm for multi-group multicast
  beamforming in large-scale wireless systems,'' {\em IEEE Trans. Commun.},
  vol.~65, pp.~2685--2698, Jun. 2017.

\bibitem{tao1997convex}
P.~D. Tao and L.~T.~H. An, ``Convex analysis approach to {DC} programming:
  Theory, algorithms and applications,'' {\em Acta Math. Vietnamica}, vol.~22,
  no.~1, pp.~289--355, 1997.

\bibitem{bouboulis2012adaptive}
P.~Bouboulis, K.~Slavakis, and S.~Theodoridis, ``Adaptive learning in complex
  reproducing kernel {H}ilbert spaces employing {W}irtinger's subgradients,''
  {\em IEEE Trans. Neural Netw. Learn. Syst.}, vol.~23, pp.~425--438, Mar.
  2012.

\bibitem{rockafellar2015convex}
R.~T. Rockafellar, {\em Convex analysis}.
\newblock Princeton university press, 2015.

\bibitem{watson1992characterization}
G.~A. Watson, ``Characterization of the subdifferential of some matrix norms,''
  {\em Linear Algebra Appl.}, vol.~170, pp.~33--45, 1992.

\bibitem{lu2017efficient}
C.~Lu and Y.-F. Liu, ``An efficient global algorithm for single-group multicast
  beamforming,'' {\em IEEE Trans. Signal Process.}, vol.~65, pp.~3761--3774,
  Jul. 2017.

\end{thebibliography}

\end{document}